\documentclass{article}

     \usepackage[nonatbib,preprint]{neurips_2019}

\usepackage[utf8]{inputenc} 
\usepackage[T1]{fontenc}    
\usepackage{hyperref}       
\usepackage{url}            
\usepackage{booktabs}       
\usepackage{amsfonts}       
\usepackage{nicefrac}       
\usepackage{microtype}      

\usepackage{amsmath,amsfonts,bm}

\def\eqref#1{equation~\ref{#1}}

\def\algref#1{algorithm~\ref{#1}}

\def\1{\bm{1}}

\def\vx{{\bm{x}}}

\def\vz{{\bm{z}}}

\DeclareMathAlphabet{\mathsfit}{\encodingdefault}{\sfdefault}{m}{sl}
\SetMathAlphabet{\mathsfit}{bold}{\encodingdefault}{\sfdefault}{bx}{n}

\newcommand{\E}{\mathbb{E}}

\DeclareMathOperator*{\argmax}{arg\,max}

\usepackage{hyperref}
\usepackage{url}
\usepackage{algorithm}
\usepackage{algorithmic}

\usepackage{amsfonts}
\usepackage{multirow}
\usepackage{amsmath}
\usepackage{graphicx}
\usepackage{amssymb}	
\usepackage{subcaption}
\usepackage{amsthm}
\usepackage{subcaption}
\usepackage{afterpage}

\newtheorem{theorem}{Theorem}

\newtheorem{lemma}{Lemma}
\newtheorem{conjecture}{Conjecture}

\newtheorem{open-question}{Open Question}

\newcommand{\lemref}[1]{Lemma~\ref{#1}}
\newcommand{\thmref}[1]{Theorem~\ref{#1}}

\usepackage{tikz}
\usetikzlibrary{decorations.pathmorphing}

\newcommand{\reals}{\mathbb{R}}

\newcommand{\mean}[2]{\mathbb{E}_{#1} \left[ #2 \right]}

\newcommand{\p}{^{\prime}}
\newcommand{\abs}[1]{\left \lvert #1 \right \rvert}

\newcommand{\todo}[1]{\textbf{TODO: #1}}

\newcommand{\mx}{\mathcal{X}}
\newcommand{\cd}{\mathcal{D}}
\newcommand{\ml}{\mathcal{L}}
\newcommand{\my}{\mathcal{Y}}

\newcommand{\ignore}[1]{}

\newcommand{\prob}[2]{\mathbb{P}_{#1}\left[#2\right]}

\newcommand{\supp}{\text{supp}}
\newcommand{\Gain}{\text{Gain}}
\newcommand{\Ber}{\text{Bernoulli}}
\newcommand{\IDt}{\text{ID3}}
\newcommand{\Ind}{\mathcal{I}}

\title{ID3 Learns Juntas for Smoothed Product Distributions}
\author{%
  Alon Brutzkus\thanks{  The Blavatnik School of Computer Science
  Tel Aviv University, Israel}
 \\
  \And
  Amit Daniely\thanks{
  School of Computer Science
  The Hebrew University, Israel}
   \And
  Eran Malach\thanks{
  School of Computer Science
  The Hebrew University, Israel}
}

\begin{document}
\maketitle

\begin{abstract}
In recent years, there are many attempts to understand popular heuristics.
An example of such a heuristic algorithm is the ID3 algorithm for learning
decision trees. This algorithm is commonly used in practice, but there are very few theoretical works studying its behavior. 
In this paper, we analyze the ID3 algorithm, when the target function is a $k$-Junta, a function that depends on $k$
out of $n$ variables of the input.
We prove that when $k = \log n$, the ID3 algorithm learns in polynomial time $k$-Juntas, in the smoothed analysis model of \cite{kalai2008decision}. That is, we show a learnability result  when the observed distribution is a ``noisy'' variant of the original distribution. 
\end{abstract}

\section{Introduction}

In recent years there has been a growing interest in analyzing machine learning
algorithms that are commonly used in practice. A primary example is the gradient-descent
algorithm for learning neural-networks, which achieves remarkable performance in practice
but has very little formal guarantees. A main approach in studying such algorithms
is proving that they are able to learn models that are known to be learnable.
For examples, it has been shown that SGD can learn neural-networks 
when the target function is linear,
or belongs to a certain kernel space \cite{brutzkus2017sgd,xie2016diverse,daniely2017sgd, du2018gradient,oymak2018overparameterized,allen2018learning,allen2018convergence,
arora2019fine, oymak_towards_2019, ma2019comparative, lee2019wide}.

In this paper we take a similar approach aiming to give theoretical guarantees
for the ID3 algorithm \cite{quinlan1986induction}
- a popular algorithm for learning decision trees.
We analyze the behavior of this algorithm when the target function is a $k$-Junta,
a function that depends only on $k$ bits from the input,
and the underlying distribution is a product distribution,
where the bits in the input examples are independent.
While we cannot guarantee that the ID3 algorithm learns under any such distribution,
as there are distributions which fail the algorithm,
we show that the algorithm can learn ``most'' such distributions.
That is, we show that for any product distribution and a $k$-Junta, the ID3 algorithm
learns the junta over a ``noisy'' variant of the original distribution.
Such a result is in the spirit of smoothed analysis \cite{spielman2004smoothed}, which is often used to give results
when a worst-case analysis is not satisfactory. 

\paragraph{Related Work} There are a number of works studying the learnability of
decision trees \cite{rivest1987learning, kushilevitz1993learning,blum1992rank,ehrenfeucht1989learning, bshouty2002using, bshouty2003proper, bshouty2005learning}. 
We next elaborate on papers that analyze decision trees under product distributions, as we do.
The work of \cite{kalai2008decision} gives learnability
results of decision trees for product distributions with smoothed analysis,
in a problem setting similar to ours. Their work analyzes an algorithm
that estimates the Fourier coefficients of the target function
in order to learn the decision tree. Another work \cite{o2007learning}
proves learnability of decision trees implementing monotone Boolean functions
under the uniform distribution.
Other algorithms for learning decision trees under the uniform distribution
are given in \cite{jackson2003learning,hazan2017hyperparameter},
again relying on Fourier analysis of the target function.
Another work \cite{chen2018beyond} gives an algorithm for learning stochastic
decision trees under the uniform distribution.
The work of \cite{blum1994weakly} gives negative results on learning polynomial
size decision trees under the uniform distribution in the statistical query
setting. 

While the above works study learnability of decision trees under various
distributional assumptions, they all consider algorithms that are very different
from algorithms used in practice.
Our work, on the contrary, gives guarantees for algorithms that enjoy 
empirical success. In the current literature there are very few works that analyze
such algorithms. Notably, the work of \cite{fiat2004decision} studies the class of
impurity-based algorithms, which contains the ID3 algorithm. This work shows that
unate functions, like linear threshold function or read-once DNF, are learnable
under the uniform distribution, using impurity-based algorithms.
Our work, on the other hand, considers a different choice of target functions
(Juntas), and shows learnability under ``most'' distributions,
and not only for a fixed distribution.
Another work that studies an algorithm used in practice \cite{kearns1999boosting}
shows that the CART and C4.5 algorithms can leverage weak approximation of the target
function, and thus can perform boosting. However, it is not clear whether such weak
approximation typically happens, and in what cases this result can be applied.
In contrast, our results apply for a concrete family of functions and distributions.

\subsection{Problem Setting}
\paragraph{The ID3 Algorithm} Let $\mathcal{X} = \{0,1\}^n$ be the domain set and let $\mathcal{Y} = \{0,1\}$
be the label set. We next describe the ID3 algorithm, following the presentation in \cite{shalev2014understanding}.
Define an impurity function $C$ to be any concave function $C:[0,1] \to \reals$,
satisfying that $C(x) = C(1-x)$ and $C(0) = C(1) = 0$.
Given an impurity function $C$, a sample $S \subset \mathcal{X} \times \mathcal{Y}$
and an index $i \in [n]$, we define the gain measure to be as follows:
\begin{align*}
\Gain(S,i) = &C(\prob{S}{y = 1}) \\
&-\left(\prob{S}{x_i = 1} C(\prob{S}{y=1 | x_i = 1})
+ \prob{S}{x_i = 0} C(\prob{S}{y=1 | x_i = 0})
\right)
\end{align*}
Given a sample $S \subseteq \mathcal{X} \times \mathcal{Y}$,
the ID3 algorithm generates a decision tree in a recursive manner. At
each step of the recursion, the algorithm chooses the feature $x_j$
to be assigned to a current node. The algorithm iterates over all the unused
features, and calculates the gain measure with respect to the examples
that reach the current node. Then, it chooses the feature that maximizes the gain.
This algorithm is described formally in \algref{alg:id3}.
The output of the algorithm is given by the initial call to $\IDt(S,[n])$.

\begin{algorithm}
   \caption{$\IDt(S,A)$}\label{alg:dc_full}
   \label{alg:id3}
\begin{algorithmic}
  \STATE \textbf{input}: 
\begin{ALC@g}
  \STATE Training set $S \subset \mathcal{X} \times \mathcal{Y}$
  \STATE Feature subset $A \subseteq [n]$
\end{ALC@g}
   \IF {all examples in $S$ have the same label $y \in \mathcal{Y}$}
   \STATE return a leaf with label $y$
   \ELSE
   \STATE Let $j=\argmax_{i \in A} \Gain(S,i)$
   \STATE Let $T_1$ be the tree returned by $\IDt(\{(\vx,y) \in S ~:~ x_j = 1\},
   A \setminus \{j\})$
   \STATE Let $T_2$ be the tree returned by $\IDt(\{(\vx,y) \in S ~:~ x_j = 0\},
   A \setminus \{j\})$
   \STATE Return the a tree with root $x_j$, whose left and right sub-trees
   are $T_2$ and $T_1$
   \ENDIF
\end{algorithmic}
\end{algorithm}

\paragraph{Learning Juntas}
A $k$-Junta is a function $f:\{0,1\}^n\to\{0,1\}$ that depends on $k$ coordinates.
Namely, there is a set $J= \{ i_1< i_2 < \ldots < i_k \} \subset [n]$ and a function
$\tilde f:\{0,1\}^k\to\{0,1\}$ such that 
$f(\vx) = \tilde f(x_{i_1},\ldots,x_{i_k})$. In this case, we will say that $f$ is supported in $J$.
Throughout the paper, we assume that the examples are sampled from a 
\textbf{product} distribution $\mathcal{D}$ over $\mathcal{X} \times \mathcal{Y}$, that is realizable by a $\log(n)$-Junta.
Namely, we assume that for $(\vx, y) \sim \mathcal{D}$,
$\vx \sim \prod_{i=1}^n \Ber(p_i)$ for some $p_1, \dots, p_n \in [0,1]$, and $y=f(\vx)$ for some $\log(n)$-Junta $f$.
 The main goal of this paper is to show that for ``most" product distributions, the ID3 algorithm succeeds to learn $\log(n)$-Juntas in polynomial time. Namely, it will return a tree $T$ whose generalization error, $\ml_\cd(T):=\Pr_{(\vx,y)\sim\cd}\left(T(\vx)\ne y\right)$, is small (in fact, zero).
We note that the sample complexity of learning $\omega\left(\log(n)\right)$-Juntas is super polynomial, hence, $\log(n)$-Juntas is the best that we can hope to learn in polynomial time.

\subsection{Results}
We will show two positive results for learning $\log(n)$-Juntas. The first establishes learnability of parities, while the second is about learnability of general Juntas. Thruought, we assume that the impurity function $C$ is strongly concave and Lipschitz.

\paragraph{Learning Parities}
A {\em $k$-parity} is a function of the form $\chi_J(\vx) = \begin{cases}1 & \sum_{i\in J}x_i\text{ is odd}\\0 & \sum_{i\in J}x_i\text{ is even} \end{cases}$, where $J\subset [n]$ is a set of $k$ indices. Note that any $k$-parity is a $k$-Junta.
We first consider leranability of $\log(n)$-parities by the ID3 algorithm\footnote{As opposed to general $k$-Juntas, $k$-parities with any $k$ are learnable in polynomial time. Yet, in the context of decision tree algorithms, we cannot hope to learn $k$-parities with $k=\omega(\log(n))$. Indeed, such parities cannot be computed, or even approximated, by a poly-sized tree.}.

Learning parity functions is a classical problem in machine learning,
for which there exists an efficient algorithm\cite{feldman2006new,feldman2009agnostic}. Still, parities often serve as a hard benchmark, as many common algorithms cannot learn these functions \cite{blum2003noise,shalev2017failures}. 
In the case of the ID3 algorithm, when the underlying distribution
is uniform (i.e, when $p_i = \frac{1}{2}$ for all $i \in [n]$), the algorithm
fails to learn parity functions \cite{kearns1996boosting}. 
We show that the case of the uniform distribution is in some sense unique.
That is, we show that for every distribution that is not ``too close''
to the uniform distribution, the ID3 algorithm succeeds to learn any such parity function.
To this end, we say that $\cd$ is $(\alpha,c)$-distributuion if $\left|p_i-\frac{1}{2}\right|>c$ and $p_i\in (\alpha,1-\alpha)$ for any $i\in [n]$.

\begin{theorem}\label{thm:parity_main}
Fix $\alpha,c>0$. There is a polynomial\footnote{The polynomial $p$ depends on $\alpha,c$ and the impurity function $C$. See theorem \ref{thm:parities} for a detailed dependency.} $p$ for which the following holds. Suppose that the ID3 algorithm runs on $p\left(n,\log\left(\frac{1}{\delta}\right)\right)$ examples from an $(\alpha,c)$-distribution $\cd$ that is realized by a $\log(n)$-parity. Then, w.p. $\ge 1-\delta$, ID3 will output a tree $T$ with $\ml_\cd(T)=0$.
\end{theorem}

\paragraph{Smoothed Analysis of Learning General Juntas}
For general Juntas, instead of standard worst-case analysis, where we require that the algorithm succeeds
to learn \textbf{any} distribution, we will show that the algorithm 
learns \textbf{most} distributions. Namely, for every fixed distribution,
we show that the algorithm succeeds to learn, with high probability, a
``noisy'' version of this distribution.
Formally, a {\em smoothened $(\alpha,c)$-distribution $\cd$} is a {\em random} distribution where  
$p_i = \hat{p}_i + \Delta_i$ for some 
$\hat{p}_i\in \left(\alpha+c, 1-\alpha-c\right)$ and $\Delta_i \sim Uni([-c,c])$.

\begin{theorem}\label{thm:junta_main}
Fix $\alpha,c>0$. There is a polynomial\footnote{The polynomial $p$ again depends on $\alpha,c$ and the impurity function $C$. See theorem \ref{thm:junta} for a detailed dependency.} $p$ for which the following holds. Suppose that the ID3 algorithm runs on $p\left(n,\frac{1}{\delta}\right)$ examples from a smoothened $(\alpha,c)$-distribution $\cd$ that is realized by a $\log(n)$-junta. Then, w.p. $\ge 1-\delta$, ID3 will output a tree $T$ with $\ml_\cd(T)=0$.
\end{theorem}

\subsection{Open Question}
We now turn to discussing possible open questions and future directions arising from
this work.
Our main result applies for the case where the target function is a $k$-Junta,
which can be implemented by a tree of depth $k = \log n$.
An immediate open question is whether a similar learnability result can be shown
for general trees of depth $\log n$. We conjecture that this is indeed the case.

\begin{conjecture}\label{thm:junta_main}
Fix $\alpha,c>0$. There is a polynomial $p$ for which the following holds. Suppose that the ID3 algorithm runs on $p\left(n,\frac{1}{\delta},\frac{1}{\epsilon}\right)$ examples from a smoothened $(\alpha,c)$-distribution $\cd$ that is realized by a $\log(n)$-depth-tree. Then, w.p. $\ge 1-\delta$, ID3 will output a tree $T$ with $\ml_\cd(T)\le \epsilon$.
\end{conjecture}

As we previously mentioned, our work could be viewed in a broader context
of understanding heuristic learning algorithms that enjoy empirical success.
In this field of research, a main challenge of the machine learning
community is to understand the behavior of neural-networks learned with gradient-based
algorithms. While our analysis is focused on proving results for the ID3 algorithm,
we believe that similar techniques could be used to show similar results for
learning neural-networks with gradient-descent.
Specifically, we raise the following interesting question:
\begin{open-question}
Can gradient-descent learn neural-networks when the target function is a $k$-Junta,
in the smoothed analysis setting?
\end{open-question}

\section{Proofs}

\subsection{General Approach}
Throughout, we assume that $\cd$ is a distribution that is realized by a Junta $f$, supported in $J\subset [n]$, with $|J| = k$. We assume w.l.o.g. that $J = [k]$.

To prove our result, we will show that w.h.p., the algorithm chooses only variables from $[k]$, and furthermore, any root-to-leaf path will contain all the variables from $[k]$. 
In this case, the resulting tree will have zero generalization error. To formalize this, we will use the following notation.
We define the support of a vector $w\in \{*,0,1\}^n$ as
\[
\supp(w) = \{i\in [n] : w_i\ne *\}
\]
and let
\[
\mx_w = \{x\in\mx : x_i = w_i\text{ for any }i \in \supp(w)\}
\]
For a sample $S \subseteq \mathcal{X} \times \mathcal{Y}$, we denote
\[
S_{w} = \{(x,y) \in S ~:~  x\in \mx_w\}
\]
Finally, for a distribution $\cd$ we denote $\cd_w = \cd |_{x\in \mx_w}$ 
\begin{lemma}\label{lem:basic}
Suppose that the sample $S$ is realized by $f$. Assume that for any $w\in \{0,1,*\}^n$ with
$\supp(w)\subset J$ we have $S_w \ne \emptyset$ and either of the following holds:
\begin{itemize}
\item All examples in $S_w$ have the same label.
\item For all $i\in J\setminus \supp(w)$ and $j\in [n]\setminus J$ we have
$\Gain(S_w,i) > \Gain(S_w,j)$
\end{itemize}
Then, the ID3 algorithm will build a tree with zero loss on $\cd$.
\end{lemma}

Not surprisingly, the gain of coordinates outside of $J$ is always small. This is formalized in the following lemma.
\begin{lemma}
\label{lem:empirical_gain_upper}
Assume that $C$ is $\gamma$-Lipschitz.
Fix $w\in \{0,1,*\}^n$, with $|\supp(w)| \le k$, $j \in [n] \setminus J$ and $\epsilon, \delta \in (0,1)$.
Assume we sample $S \sim \mathcal{D}^m$ with
$m \gtrsim  \epsilon^{-2}\alpha^{-2k} \log(\frac{1}{\delta})$.
Then with probability at least $1-\delta$ we have $S_w \ne \emptyset$ and:
\[
Gain(S_w,j) < 2 \gamma \epsilon
\]
\end{lemma}
Given lemma \ref{lem:empirical_gain_upper}, in order to apply lemma \ref{lem:basic}, it remains to show that the gain of the coordinates in $J$ is large. To this end, we will use a measure of dependence between a coordinate $x_i$ and the label $y$, which we define next.
For a sample $S \subset \mathcal{X} \times \mathcal{Y}$ and an index $i \in [n]$,
we let $\Ind(S,i) = \mean{S}{y}\mean{S}{x_i}- \mean{S}{yx_i}$.
Similarly, for a distribution $\mathcal{D}$ over $\mathcal{X} \times \mathcal{Y}$
we let $\Ind(\mathcal{D},i) = \mean{\mathcal{D}}{y}\mean{\mathcal{D}}{x_i}
- \mean{\mathcal{D}}{yx_i}$.
Note that $x_j$ and $y$ are independent if and only if $\Ind(\mathcal{D},i) = 0$. The following lemma connects $\Gain(S_w,i)$ to $\Ind(\cd_w,i)$.

\begin{lemma}
\label{lem:empirical_gain_lower_bound}
Assume $C$ is $\beta$ strongly concave (i.e, $-C$ is $\beta$
strongly convex). Assume for some $w\in \{0,1,*\}^n$, with $\supp(w)\le k$ and index $i \in [n]$
we have $|\Ind(\mathcal{D}_w,i)| \ge \epsilon > 0$.
Fix $\delta > 0$. Then, if we sample $S \sim \mathcal{D}^m$ for
$m \gtrsim \epsilon^{-2}\alpha^{-2k} \log(\frac{1}{\delta})$,
then with probability at least $1-\delta$ we have $S_w \ne \emptyset$ and:
\[
\Gain(S_w,i) \ge \frac{\beta \epsilon^2}{8}
\]
\end{lemma}

Combining lemmas \ref{lem:basic}, \ref{lem:empirical_gain_upper} and \ref{lem:empirical_gain_lower_bound}, we get the following theorem:
\begin{theorem}\label{thm:basic}
Assume $C$ is $\beta$ strongly concave and $\gamma$-Lipschitz. Assume for any $w\in \{0,1,*\}^n$, with $\supp(w)\subset J$ we have $S_w \ne \emptyset$ and either of the following holds:
\begin{itemize}
\item All examples in $\mathcal{D}_w$ have the same label.
\item For every index $i \in J\setminus \supp(w)$ we have $|\Ind(\mathcal{D}_w,i)| \ge \epsilon > 0$.
\end{itemize}
Fix $\delta > 0$. Then, if we sample $S \sim \mathcal{D}^m$ for
$m \gtrsim \beta^{-2}\gamma^{2} \epsilon^{-4}\alpha^{-2k}k \log(\frac{n}{\delta})$,
then with probability at least $1-\delta$ the ID3 algorithm will build a tree with zero loss on $\cd$
\end{theorem}

By the above theorem, in order to show that the ID3 algorithm succeeds in learning, it is enough to lower bound $|\Ind(\mathcal{D}_w,i)|$. This is done in the remaining sections, together with the proof of lemmas \ref{lem:basic}, \ref{lem:empirical_gain_upper} and \ref{lem:empirical_gain_lower_bound}.

\subsection{Proof of the basic lemmas}
\begin{proof} (of lemma \ref{lem:basic})
At every iteration, the ID3 algorithm assigns a splitting variable
for a given node, or otherwise returns a leaf for this node.
We will show that for every node that the algorithm iterates on,
if the path from the root to this node contains only variables from $J$,
then either the algorithm adds a splitting variable from $J$,
or the algorithm returns a leaf.
Indeed, assume that the path from the root to this node contains only variables from $J$.
We can decode the root-to-node path by a vector $w \in \{*,0,1\}^n$,
where $w_i = 1$ if the node $x_i = 1$ is in the path,
$w_i = 0$ if the node $x_i = 0$ is in the path,
and $w_i = *$ otherwise.
Therefore, by our assumption we have $\supp(w) \subseteq J$.
Note that in this case, the algorithm observes the sample $S_w$,
so if all examples in $S_w$ have the same label, then the algorithm returns a leaf.
Otherwise, by the assumption we get $\argmax_{i \in A} Gain(S_w,i) \in J$,
so the algorithm chooses a splitting variable from $J$.

From the above, the algorithm adds only splitting variables from $J$,
so it can build a tree of size at most $2^k$ before stopping.
This tree has zero loss on the distribution.
Indeed, for any $x' \in \{0,1\}^k$, denote $w(x') \in \{0,1\}^n$
such that $w(x')_i = x'_i$ for every $i \in [k]$ 
and $w(x')_i = *$ for every $i \notin [k]$.
Then, since we assume $S_{w(x')} \ne \emptyset$,
there exists a sample $(x,y) \in S$ such that $x_i = x'_i$ for every $i \in [k]$.
By definition, the algorithm returns a tree that correctly labels the example
$x$, therefore it returns a function that agrees with
$f(x) = \tilde{f}(x')$. Since this is true for every choice of $x' \in \{0,1\}^k$,
the function returned by the tree agrees with the Junta defined by $\tilde{f}$,
so it gets zero loss.
\end{proof}

We next relate the empirical measure $\Ind(S,i)$ to $\Ind(\mathcal{D},i)$.

\begin{lemma}
\label{lem:empirical_correlation}
Fix $w\in \{0,1,*\}^n$, with $\supp(w)\le k$, $i \in [n]$, $\epsilon,\delta \in (0,1)$. Let $S \sim \mathcal{D}^m$
with $m \gtrsim \alpha^{-2k}\epsilon^{-2} \log(\frac{1}{\delta}) $.
Then with probability at least $1-\delta$ we have $S_w \ne \emptyset$ and:
\[
\abs{\Ind(S_w,i) - \Ind(\mathcal{D}_w,i)} < \epsilon
\]
\end{lemma}

\begin{proof}
Denote $S = \{(\vx_1, y), \dots, (\vx_m, y)\}$. Let $\bar{p_i} = \mean{S_w}{x_i}$ and $p_w = \Pr_{x\sim\cd}\left(x\in \mx_w\right)$.
We have
\[
\mean{S_w}{y} = \frac{\sum_{j=1}^m 1[\vx_j\in \mx_w ] y_j  }{\sum_{j=1}^m 1[\vx_j\in \mx_w ]} = \frac{\frac{\sum_{j=1}^m 1[\vx_j\in \mx_w ] y_j}{p_w m}  }{\frac{\sum_{j=1}^m 1[\vx_j\in \mx_w ]}{p_w m}}
\]
By Hoeffding's bound, with probability $\ge 1-\frac{\delta}{3}$, we have
\[
\left| p_w\E_{\cd_w}y - \frac{\sum_{j=1}^m 1[\vx_j\in \mx_w ] y_j}{m}  \right| \lesssim \epsilon \alpha^k \text{ and }\left| p_w - \frac{\sum_{j=1}^m 1[\vx_j\in \mx_w ]}{m}  \right| \lesssim \epsilon \alpha^k
\]
dividing by $p_w$ we get
\[
\left| \E_{\cd_w}y - \frac{\sum_{j=1}^m 1[\vx_j\in \mx_w ] y_j}{p_wm}  \right| \lesssim \frac{\epsilon \alpha^k }{p_w} \lesssim \epsilon \text{ and }\left| 1 - \frac{\sum_{j=1}^m 1[\vx_j\in \mx_w ]}{p_wm}  \right| \lesssim \frac{\epsilon \alpha^k }{p_w} \lesssim \epsilon
\]
Notice that from the above we get that $S_w \ne \emptyset$.
It follows that
\[
|\mean{S_w}{y} - \E_{\cd_w}y| \lesssim \epsilon
\]
Similarly,
\[
|\mean{S_w}{x_i} - \E_{\cd_w}x_i| \lesssim \epsilon \text{ and } | \mean{S_w}{yx_i}- \E_{\cd_w}yx_i| \lesssim \epsilon
\]
In this case, we have $\abs{\Ind(S_w,i) - \Ind(\mathcal{D}_w,i)} < \epsilon$
\end{proof}

We next prove lemmas \ref{lem:empirical_gain_upper} and \ref{lem:empirical_gain_lower_bound}

\begin{proof} (of lemma \ref{lem:empirical_gain_upper})
Notice that since $x_i$ and $y$ are independent, we have $\Ind(\mathcal{D}_w,i) = 0$.
By the choice of $m$, from
\lemref{lem:empirical_correlation} we get that with probability $1-\delta$:
\[
\abs{\Ind(S_w,i)} =
\abs{\Ind(S_w,i) - \Ind(\mathcal{D}_w,i)} < \epsilon
\]
Denote $\bar{p_i} = \prob{S_w}{x_i = 1}$.
Notice that if $\bar{p_i} = 0$ or $\bar{p_i} = 1$ then $\Ind(S_w, i) = 0$,
and the result trivially holds. We can therefore assume $\bar{p_i} \in (0,1)$.
Now, we have the following:
\begin{align*}
\abs{\prob{S_w}{y=1 | x_i=1} - \prob{S_w}{y=1}}
&=
\abs{\frac{\prob{S_w}{y=1 \wedge x_i = 1} - \prob{S_w}{y = 1}\prob{S_w}{x_i = 1}}{\prob{S_w}{x_i = 1}}} \\
& = \abs{\frac{\Ind(S_w,i)}{\bar{p_i}}} < \frac{\epsilon}{\bar{p_i}}
\end{align*}
Similarly, we get:
\begin{align*}
\abs{\prob{S_w}{y=1 | x_i=0} - \prob{S_w}{y=1}}
&=
\abs{\frac{\prob{S_w}{y=1 \wedge x_i = 0} - \prob{S_w}{y = 1}\prob{S_w}{x_i = 0}}{\prob{S_w}{x_i = 0}}} \\
& = \abs{\frac{\Ind(S_w,i)}{1-\bar{p_i}}} < \frac{\epsilon}{1-\bar{p_i}}
\end{align*}
Using the $\gamma$-Lipschitz property, we get:
\[
\abs{C(\prob{S_w}{y=1 | x_i=1}) - C(\prob{S_w}{y=1})} \le 
\gamma \abs{\prob{S_w}{y=1 | x_i=0} - \prob{S_w}{y=1}} < \frac{\gamma \epsilon}{\bar{p_i}}
\]
And similarly:
\[
\abs{C(\prob{S_w}{y=1 | x_i=0}) - C(\prob{S_w}{y=1})} < \frac{\gamma \epsilon}{1-\bar{p_i}}
\]
Now plugging into the gain definition:
\begin{align*}
\abs{Gain(S_w,i)}
=& |C(\prob{S_w}{y=1})
- (\prob{S_w}{x_i=1} C(\prob{S_w}{y=1 | x_i=1}) \\
&+\prob{S_w}{x_i=0} C(\prob{S_w}{y=1 | x_i=0}))| \\
\le & \prob{S_w}{x_i=1} \abs{C(\prob{S_w}{y=1 | x_i=1})- C(\prob{S_w}{y=1})} \\
&+\prob{S_w}{x_i=0} \abs{C(\prob{S_w}{y=1 | x_i=0})- C(\prob{S_w}{y=1})} 
< 2\gamma \epsilon
\end{align*}
\end{proof}

\begin{proof} (of lemma \ref{lem:empirical_gain_lower_bound})
By the choice of $m$, from
\lemref{lem:empirical_correlation} we get that with probability $1-\delta$:
\[
|\Ind(S_w,i) - \Ind(\mathcal{D}_w,i)| \le \frac{\epsilon}{2}
\]
Since we assume $|\Ind(\mathcal{D}_w,i)| \ge \epsilon$, we get that
$|\Ind(S_w,i)| \ge \frac{\epsilon}{2}$.
Therefore, we have that $\bar{p_i} \in (0,1)$
(again denoting $\bar{p_i} = \prob{S_w}{x_i=1}$).
Observe that we have the following:
\begin{align*}
&\prob{S_w}{x_i=1}\prob{S_w}{x_i=0}
(\prob{S_w}{y = 1 | x_i = 0} - \prob{S_w}{y = 1 | x_i = 0}) \\
&= 
\prob{S_w}{x_i=0}\prob{S_w}{x_i=1 \wedge y = 1} -
\prob{S_w}{x_i=1}\prob{S_w}{x_i=0 \wedge y = 1} \\
&= 
\prob{S_w}{x_i=0}\prob{S_w}{x_i=1 \wedge y = 1} -
\prob{S_w}{x_i=1}(\prob{S_w}{y=1} - \prob{S_w}{x_i=1 \wedge y = 1}) \\
&= \prob{S_w}{x_i=1 \wedge y = 1} -
\prob{S_w}{x_i=1}\prob{S_w}{y=1} = \Ind(S_w,i)
\end{align*}
Therefore, we have:
\[
\prob{S_w}{y=1 | x_i=1} - \prob{S_w}{y=1 | x_i=0}
= \frac{\Ind(S_w,i)}{\bar{p_i}(1-\bar{p_i})}
\]
Since $C$ is $\beta$ strongly concave we get that for all $a,b,t \in [0,1]$ we have:
\[
C(ta + (1-t)b) \ge t C(a) + (1-t)C(b) + \frac{\beta}{2} t (1-t) (a-b)^2
\]
Using this property we get that:
\begin{align*}
&\bar{p_i} C(\prob{S_w}{y=1 | x_i=1})
+ (1-\bar{p_i}) C(\prob{S_w}{y=1 | x_i=0}) \\
&\le C(\prob{S_w}{y=1}) 
- \frac{\beta}{2}  \bar{p_i}
(1-\bar{p_i}) (\prob{S_w}{y=1|x_i=1} - \prob{S_w}{y=1|x_i=0})^2 \\
&= C(\prob{S_w}{y=1})-\frac{\beta}{2} \cdot \frac{\Ind(S_w,i)^2}{\bar{p_i}(1-\bar{p_i})}
\end{align*}
Plugging this to the gain equation we get:
\begin{align*}
Gain(S_w,i) \ge \frac{\beta}{2} \cdot \frac{\Ind(S_w,i)^2}{\bar{p_i}(1-\bar{p_i})}
\ge \frac{\beta}{2} \Ind(S_w,i)^2
\end{align*}
Since $|\Ind(S_w,i)| \ge \frac{\epsilon}{2}$, we get
$Gain(S_w,i) \ge \frac{\beta \epsilon^2}{8}$.

\end{proof}

\subsection{Parities}

\begin{lemma}
\label{lem:parity_I_lower_bound}
Let $\cd$ be a distribution on $\mx\times\my$ labelled by $\chi_J$ with $|J|\le k$.
Assume that for every $j \in J$ we have
$p_j \in (\alpha, 1-\alpha)$ and $|p_j-\frac{1}{2}| \ge c$,
for some $c, \alpha > 0$.
Fix some $w \in \{0,1,*\}^k$.
Then for every $j \in J \setminus \supp(w)$ we have:
\begin{align*}
|\Ind(\mathcal{D}_w,j)| > \alpha^2 (2c)^{k-1}
\end{align*}
\end{lemma}

By theorem \ref{thm:basic} we have

\begin{theorem}\label{thm:parities}
Let $\cd$ be a distribution on $\mx\times\my$ labelled by $\chi_J$ with $|J|\le k$.
Assume that for every $j \in J$ we have
$p_j \in (\alpha, 1-\alpha)$ and $|p_j-\frac{1}{2}| \ge c$,
for some $c, \alpha > 0$. Assume furthermore that $C$ is $\beta$ strongly concave and $\gamma$-Lipschitz. 

Then, if we sample $S \sim \mathcal{D}^m$ for
$m \gtrsim \beta^{-2}\gamma^{2} (2c)^{-4k-4}\alpha^{-2k-8}k \log(\frac{n}{\delta})$,
then with probability at least $1-\delta$ the ID3 algorithm will build a tree with zero loss on $\cd$
\end{theorem}

Note that since we assume $k \le \log n$, the runtime and sample complexity
in the above theorem are polynomial in $n$. We give the proof of this theorem
in the rest of this section.

\begin{proof}
Denote $\epsilon_i := p_i - \frac{1}{2}$, and $k' := |A|$.
For simplicity of notation, assume w.l.o.g that $A = [k']$ and $j = k'$.
Observe the following:
\begin{align*}
\prob{\mathcal{D}_w}{y = 1}
= \prob{\mathcal{D}_w}{\prod_{i = 1}^{k'} (2x_i-1) = 1}
= \sum_{\prod z_i = 1}
\prod_{i=1}^{k'} \prob{}{2x_i-1 = z_i}
= \frac{1}{2} + 2^{k'-1} \prod_{i=1}^{k'} \epsilon_i
\end{align*}
Similarly, we get that:
\begin{align*}
\prob{\mathcal{D}_w}{y = 1 | x_{k'} = 1}
= \sum_{\prod z_i = 1}\prod_{i=1}^{k'-1} (\frac{1}{2} + z_i \epsilon_i)
= \frac{1}{2} + 2^{k'-2} \prod_{i=1}^{k'-1} \epsilon_i
\end{align*}
Therefore, we get that:
\begin{align*}
\abs{\Ind(\mathcal{D}_w,j)}
&= p_j \abs{\prob{\mathcal{D}_w}{y=1} - \prob{\mathcal{D}_w}{y=1 | x_i = 1}} \\
&= p_j \abs{\epsilon_{k'}-\frac{1}{2}}
\cdot (2^{k'-1} \prod_{i=1}^{k'-1} \abs{\epsilon_i})
 \ge \alpha^2 (2c)^{k'-1} \ge \alpha^2 (2c)^{k-1}
\end{align*}
\end{proof}

\subsection{Juntas}

\begin{lemma}
\label{lem:juntas_I_lower_bound}
Fix some $w \in \{*,0,1\}^k$, and assume not all examples in $\mathcal{D}_w$
have the same label. Let $A = \{i \in [k] ~:~ w_i = *\}$.
Assume $p_i \in (\alpha, 1-\alpha)$ for $\alpha > 0$ for every $i$,
and fix $\delta > 0$.
Then there exists $i \in A \cap [k]$ such that with probability $1-\delta$ over the choice of $\Delta$:
 \[
|\Ind(\mathcal{D}_w,i)| > 
2 \alpha^2  \delta^2 \left(\frac{c}{2}\right)^{2k}
\]
\end{lemma}

By theorem \ref{thm:basic} we get

\begin{theorem}\label{thm:junta}
Assume $C$ is $\beta$ strongly concave and $\gamma$-Lipschitz. 
Fix $\delta_1,\delta_2 > 0$. Then, if we sample $S \sim \mathcal{D}^m$ for
$m \gtrsim \beta^{-2}\gamma^{2} c^{-8k}\delta_1^{-8}\alpha^{-2k-8}k \log(\frac{n}{\delta_2})$,
then with probability at least $1-\delta_1-\delta_2$ the ID3 algorithm will build a tree with zero loss on $\cd$
\end{theorem}

\begin{proof}
For simplicity of notation, we assume w.l.o.g. that $A=[k']$ for some $k' \le k$.
Denote $f_w : \{0,1\}^{k'} \to \{0,1\}$, such that
$f_w(x_1, \dots, x_{k'}) = f(x_1, \dots, x_{k'}, w_{k'+1}, \dots, w_k)$.
Observe the Fourier coefficients of $f_w$:
\[
f_w(\vx) = \sum_{I \subset [k']} \alpha_I \chi_{I}(\vx)
\]
Where $\chi_{I} = \prod_{i \in I} (2x_i-1)$, and note that $\chi_I$ is a Fourier basis (w.r.p to the unifrom distribution).
Notice that $|\alpha_I| \ge \frac{1}{2^k}$ for every $\alpha_I \ne 0$.
Indeed, we have:
\[
\alpha_I = \mean{\vx \sim U(\{0,1\}^k)}{\chi_I (\vx) f(\vx)} = 
\frac{1}{2^k} \sum_{\vx \in \{0,1\}^k} \chi_I(\vx) f(\vx)
\]
where $\chi_I(\vx) f(\vx) \in \{-1,0,1\}$, and this gives the required.
Since not all examples in $\mathcal{D}_w$ have the same label,
we know that $f_w$ is not a constant function.
Therefore, there exists $\emptyset \ne I_0 \subseteq [k']$ such that
$\alpha_{I_0} \ne 0$. Fix some $i \in I_0$, and we assume w.l.o.g. that $i=1$
(so $1 \in I_0$).
Now, we can write:
\[
f_w(x_1, \dots, x_{k'}) = (2x_1-1) g_w(x_2, \dots, x_{k'})
+ h_w(x_2, \dots, x_{k'})
\]
Where:
$g(x_2, \dots, x_{k'}) = \sum_{I \subset [k'], 1 \in I} \alpha_I
\chi_{I \setminus \{1\}}(\vx)$.

and since $\alpha_{I_0} \ne 0$ and $1 \in I_0$ we get $g \ne 0$.
Now, notice that since $\vx \in \{0,1\}^n$ we get:
\begin{align*}
\mean{\mathcal{D}_w}{f(\vx)} =
\mean{\mathcal{D}}{f(\vx) | x_{k'+1} = w_{k'+1}, \dots, x_k = w_k}
= \mean{\mathcal{D}}{f_w(x_1, \dots, x_{k'})}
= f_w(p_1, \dots, p_{k'})
\end{align*}
And similarly: $\mean{\mathcal{D}_w}{f(\vx)|x_1 = 1} = f_w(1,p_2, \dots, p_{k'})$.

Therefore we get:
\begin{align*}
|\mean{\mathcal{D}_w}{f_w(\vx)x_1} - \mean{\mathcal{D}_w}{f_w(\vx)}
\mean{\mathcal{D}_w}{x_1}|
&= |p_1 f_w(1,p_2, \dots, p_{k'}) - p_1 f_w(p_1, \dots, p_{k'})| \\
&= p_1 |g_w(p_2, \dots, p_{k'}) - (2p_1-1) g_w(p_2, \dots, p_{k'})| \\
&= 2p_1(1-p_1) |g_w(p_2, \dots, p_{k'})| \\
&= 2p_1(1-p_1) |g_w(\hat{p}_2+\Delta_2, \dots, \hat{p}_{k'}+\Delta_{k'})| \\
&= 2p_1(1-p_1) |g_0(\Delta_2, \dots, \Delta_{k'})| \\
\end{align*}
Where $g_0$ is given by:
\begin{align*}
g_0(\Delta_2, \dots, \Delta_{k'}) &= 
g_w(\hat{p}_2+\Delta_2, \dots, \hat{p}_{k'}+\Delta_{k'}) \\
&= \sum_{I \subset [k'], 1 \in I} \alpha_I \prod_{i \in I \setminus \{1\}}
(2\hat{p_i} + 2 \Delta_i -1) \\
&= \sum_{I \subseteq [k'], 1 \in I} \alpha_I \sum_{I' \subset I}
\prod_{i \in I'} (2\hat{p}_i - 1) \prod_{j \notin I'} (2\Delta_j) 
:= \sum_{I \subseteq [k'], 1 \in I} \beta_I \prod_{i \in I \setminus \{1\}} \Delta_i
\end{align*}

Denote $k_0 = \deg(p_0)$ and note that $k_0 \le k'-1$.
For some choice of $\beta_I$-s. Notice that for some maximal $I \subset [k']$
with $1 \in I$ and $\alpha_I \ne 0$ (so $|I| = k_0$),
we have $\beta_I = 2^{|I|} \alpha_I$, so $|\beta_I| \ge \frac{2^{k_0}}{2^{k'}}$.

Now, denote $\xi_i = \frac{1}{c}\Delta_i$, so we have $\xi_i \sim Uni([-1,1])$,
and observe the polynomial:
\begin{align*}
G_0(\xi_2, \dots, \xi_{k'}) &= \frac{2^{k'}}{2^{k_0}c^{k_0-1}} g_0(c\xi_2, \dots, c\xi_k) \\
&= \sum_{I \subseteq [k'], 1 \in I} \frac{2^{k'}}{2^{k_0}c^{k_0-1}}
\beta_I c^{|I|-1} \prod_{i \in I \setminus \{1\}} \xi_i 
:= \sum_{I \subseteq [k'], 1 \in I} \gamma_I \prod_{i \in I\setminus \{1\}} \xi_i
\end{align*}

And from what we have shown, $G_0$ is a polynomial of degree $k_0$,
and there exists $I$ with $|I| = k_0$ such that $|\gamma_I| \ge 1$.
Therefore, we can use Lemma 3 from \cite{kalai2008decision} to get that:
\[
\prob{\xi \sim Uni([-1,1]^{k'})}{|G_0(\xi)| \le \epsilon} \le 2^{k_0} \sqrt{\epsilon}
\]
And therefore:
\begin{align*}
\prob{\Delta \sim Uni([-c,c]^{k'})}{|g_0(\Delta)| \le \epsilon}
&=\prob{\xi \sim Uni([-1,1]^{k'})}{|G_0(\xi)| \le
\frac{2^{k'}}{2^{k_0}c^{k_0-1}}\epsilon} \\
&\le 2^{k_0} \frac{2^{k'/2}}{2^{k_0/2}c^{k'/2-1/2}} \sqrt{\epsilon} 
\le \left(\frac{2}{c}\right)^{k'} \sqrt{\epsilon}
\le \left(\frac{2}{c}\right)^{k} \sqrt{\epsilon}
\end{align*}

So if we take $\epsilon = \delta^2 \left(\frac{c}{2}\right)^{2k}$
we get that $\prob{}{|g_0| \le \epsilon} \le \delta$, which completes the proof.

\end{proof}

\bibliography{decision_trees_juntas}
\bibliographystyle{plain}

\end{document}